\documentclass[letterpaper, 10 pt, conference]{ieeeconf}  %

\IEEEoverridecommandlockouts
\usepackage{amsthm}

\usepackage{amsmath,amssymb,amsfonts}
\usepackage{mathtools}
\DeclareMathAlphabet\mathbfcal{OMS}{cmsy}{b}{n}
\usepackage{bm}
\usepackage{xcolor}
\usepackage{soul}
\usepackage{siunitx}

\newtheorem{theorem}{Theorem}

\newtheorem{proposition}{Proposition}
\newtheorem{remark}{Remark}

\providecommand{\R}{\ensuremath \mathbb{R}}
\providecommand{\Z}{\ensuremath \mathbb{Z}}

\newcommand{\defeq}{\vcentcolon=}

\newcommand{\xf}{x^{\mathrm{f}}}
\newcommand{\yf}{y^{\mathrm{f}}}
\newcommand{\byf}{\mathbf{y^{\mathrm{f}}}}
\newcommand{\Bd}{\overline{B}}
\newcommand{\Cd}{\overline{C}}
\newcommand{\Sd}{S_d}
\newcommand{\sel}{S_{\mathrm{sel}}}
\newcommand{\Sx}{S_x}

\newcommand{\nd}{n_y}

\newcommand{\bu}{\mathbf{u}}
\newcommand{\ubar}{\bar{u}}
\newcommand{\bubar}{\mathbf{\bar{u}}}

\newcommand{\xhat}{\hat{x}}
\newcommand{\bxhat}{\mathbf{\hat{x}}}
\newcommand{\xbar}{\bar{x}}
\newcommand{\bxbar}{\mathbf{\bar{x}}}

\newcommand{\bd}{\mathbf{d}}
\newcommand{\dhat}{\hat{d}}
\newcommand{\bdhat}{\mathbf{\hat{d}}}

\usepackage{cite}
\usepackage{hyperref}
\usepackage{graphicx}
\usepackage{textcomp}
\usepackage{xcolor}
\usepackage{algpseudocode}
\usepackage{algorithm}
\usepackage[skip=0pt,font=footnotesize]{caption}

\usepackage{comment}

\usepackage{setspace}

\title{\LARGE \bf Perfecting Periodic Trajectory Tracking: \\ Model Predictive Control with a Periodic Observer ($\Pi$-MPC)\vspace{-2mm}}

\author{Luis Pabon$^{1}$, Johannes K\"ohler$^{2}$, John Irvin Alora$^{1}$, Patrick Benito Eberhard$^{2}$, Andrea Carron$^{2}$, \\Melanie N. Zeilinger$^{2}$, Marco Pavone$^{1}$%
\thanks{$^{1}$Department of Aeronautics and Astronautics, Stanford University, Stanford, CA
        (e-mail: \{lpabon, jjalora, pavone\}@stanford.edu).}%
\thanks{$^{2}$Institute for Dynamic Systems and Control, ETH Zürich, Zürich CH-8092, Switzerland
        (e-mail: \{jkoehle, peberhard, carrona, mzeilinger\}@ethz.ch).}%
\thanks{Johannes K\"ohler was supported by the Swiss National Science Foundation under NCCR Automation (grant agreement 51NF40 180545).}%
\thanks{This work was supported by the Office of Naval Research under Grant N00014-23-S-B001.} %
}

\begin{document}

\makeatletter
\g@addto@macro\@maketitle{
  \captionsetup{type=figure}\setcounter{figure}{0}
  \def\mycolspace{1.2mm}
  \centering
    \includegraphics[width=0.95\textwidth]{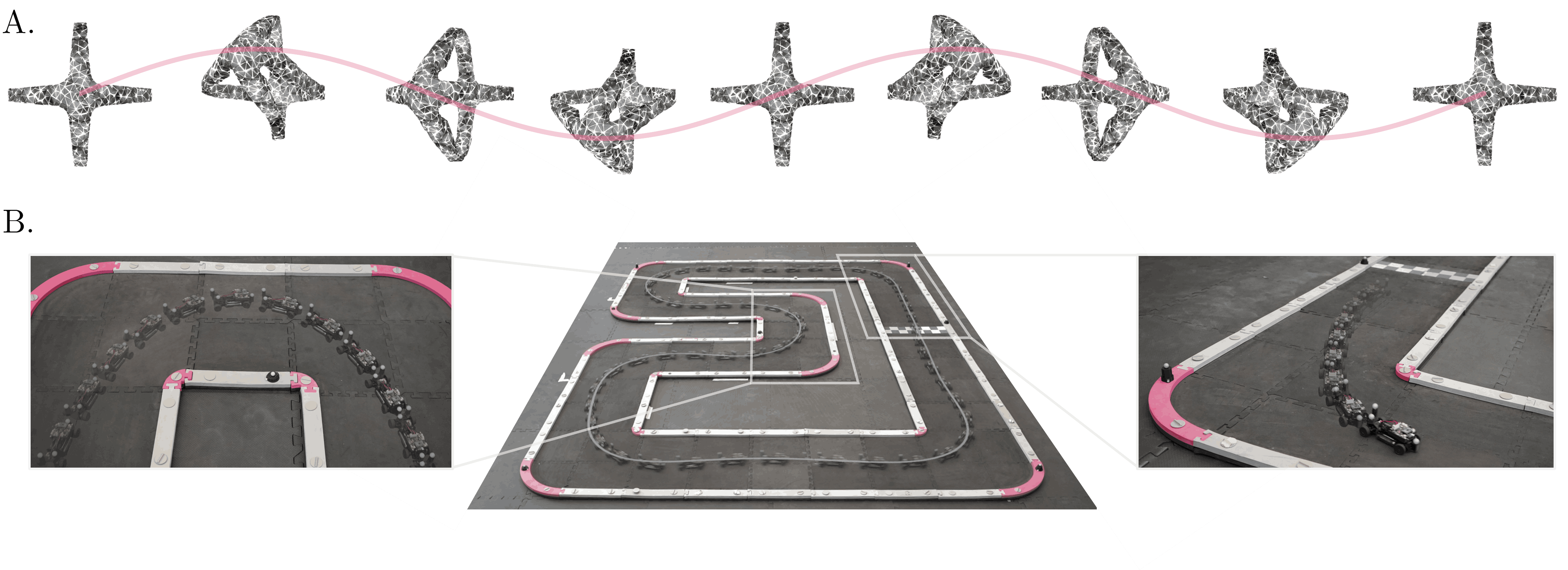} \vspace{-5mm}
    \captionof{figure}{We propose to augment MPC with a periodic observer to achieve (almost) perfect tracking despite significant model mismatch. 
    A: Sequential snapshots of a 9768-dim. Finite Element Method (FEM) simulation %
    of a ``Diamond" soft robot tracking a figure-8 with our approach. The tracking error decays to zero asymptotically despite significant model discrepancy and reduces below $1 \times 10^{-2}$ \si{mm} after 50 periods (25 seconds).
    B: Time-lapse view of a miniature race car following a reference trajectory using our proposed approach, which reduces the baseline MPC peak error from $14$ \si{cm} to $2.9$ \si{cm}.
	\label{fig:titlefig}}
 \vspace{-3.5mm}
}
\makeatother
\maketitle
\thispagestyle{empty}
\pagestyle{empty}

\begin{abstract}

In Model Predictive Control (MPC), discrepancies between the actual system and the predictive model can lead to substantial tracking errors and significantly degrade performance and reliability. While such discrepancies can be alleviated with more complex models, this often complicates controller design and implementation. By leveraging the fact that many trajectories of interest are periodic, we show that perfect tracking is possible when incorporating a simple observer that estimates and compensates for periodic disturbances. We present the design of the observer and the accompanying tracking MPC scheme, proving that their combination achieves zero tracking error asymptotically, regardless of the complexity of the unmodelled dynamics. We validate the effectiveness of our method, demonstrating asymptotically perfect tracking on a high-dimensional soft robot with nearly 10,000 states and a fivefold reduction in tracking errors compared to a baseline MPC on small-scale autonomous race car experiments.

\end{abstract}
\section*{Supplementary Material}
\noindent\textbf{Code: }\href{https://github.com/StanfordASL/Pi-MPC}{https://github.com/StanfordASL/Pi-MPC}

\noindent\textbf{Video: }\href{https://youtu.be/vBgiodXCQVQ}{https://youtu.be/vBgiodXCQVQ}

\section{Introduction}
Model Predictive Control (MPC)\cite{rawlings2017model} is a state-of-the-art method for reference tracking due to its effectiveness in enforcing safety constraints and handling nonlinear models. However, tracking performance is limited by the accuracy of the nominal model used to predict the actual system's behavior. This prediction model is typically an approximation of the true and unknown underlying dynamics. The resulting model mismatch can lead to significant tracking errors and makes achieving perfect tracking with MPC challenging, if not impossible.

Recent works have addressed this issue by learning the residual dynamics of a system with deep neural networks \cite{shi2019}, Gaussian processes \cite{kabzan2019, torrente2021}, or other data-driven methods. With enough data, these approaches can achieve high tracking accuracy. However, optimizing over complex, expressive, nonlinear models increases the computational burden and can complicate the controller design and implementation.

Repetitive tasks, and hence periodic trajectory tracking, play a crucial role across a broad spectrum of applications in robotics and control. Representative examples include legged locomotion \cite{holmes2006dynamics}, industrial manipulation \cite{cosner1990}, and autonomous racing \cite{romero2022}. %
Leveraging the periodic nature of these tasks presents an opportunity to achieve perfect tracking without the need for complex, data-driven models.

A similar observation has been made in the context of setpoint tracking. Offset-free MPC schemes \cite{badgwellDisturbanceModelDesign2002, pannocchia2003, maederLinearOffsetfreeModel2009, pannocchiaOffsetfreeMPCExplained2015} %
`learn' only what is necessary to achieve the control task. The key idea is to augment the model and use a disturbance observer that estimates a constant offset to account for steady-state error. Thus, despite using a simplified model, the MPC can achieve exact convergence -- but only to a desired setpoint.

\textbf{Statement of Contributions:} In this work, we propose an extension of offset-free MPC that asymptotically achieves zero tracking error for general periodic reference signals, i.e., perfect tracking, despite a large model mismatch. Our contributions are as follows:
\begin{enumerate}
    \item We present the design of a linear observer to estimate a periodic disturbance that captures model mismatch throughout the period. The observer ensures that the model's output predictions match measurements from the real system upon convergence.
    \item We incorporate these estimated disturbances in a simple tracking MPC and provide sufficient conditions to theoretically guarantee that the scheme achieves zero tracking error asymptotically. %
    \item While the initial presentation considers linear prediction models, we show how the method can be applied to nonlinear models and formulate a simple nonlinear MPC scheme that also achieves exact periodic tracking. 
    \item Lastly, we validate our approach through
    \begin{enumerate}
        \item Finite Element Method (FEM) simulations on a 9768-dim. soft robot using an MPC based on a learned 6-dim. linear model (Fig.~\ref{fig:titlefig}A),
        \item hardware experiments on a miniature race car using nonlinear MPC based on a simple kinematic bicycle model (Fig.~\ref{fig:titlefig}B).
    \end{enumerate}
    Despite the use of simple models, our method consistently achieves significantly lower tracking errors, reducing them by factors ranging from 4 to over 5000 compared to baseline MPC.
\end{enumerate}

\textbf{Related Work:} The field of control theory has extensively explored the estimation and rejection of periodic disturbances. Frameworks like Iterative Learning Control (ILC) \cite{ahn2007} and Repetitive Control (RC) \cite{li2004} can improve tracking accuracy by `learning' from past errors \cite{wang2009}. ILC is tailored for scenarios where systems undergo a state reset with each new operation cycle, whereas RC is designed for systems continuously transitioning across cycles.

Following the Internal Model Principle~\cite{francis1975}, RC incorporates a periodic signal generator in the controller, allowing it to reject periodic disturbances. %
Formulations combining RC and MPC were proposed in \cite{lee2001, cao2009}, showing success on periodically time-varying linear systems. However, RC and ILC directly utilize the measured error from the last cycle to update the prediction model, which can lead to poor performance due to non-repeating errors, e.g., from measurement noise~\cite{li2021performance}.

In contrast, offset-free MPC methods \cite{badgwellDisturbanceModelDesign2002, pannocchia2003, maederLinearOffsetfreeModel2009, pannocchiaOffsetfreeMPCExplained2015} avoid such pitfalls by using a more general disturbance observer to filter deterministic disturbances caused by model mismatch. The design ensures zero steady-state error and can balance noise suppression and convergence rate by tuning the observer. While this method is successfully used in many implementations \cite{carron2019data, chen2022}, its focus is mainly on setpoint tracking.

From a technical perspective, our work is related to \cite{maederOffsetfreeReferenceTracking2010}, which generalizes offset-free MPC methods to references generated by arbitrary, unstable dynamics. By focusing on periodic problems, we can provide a simpler parametrization and design for the observer (cf. \eqref{eq:LDO} and \cite[Eq. (12)]{maederOffsetfreeReferenceTracking2010}) and an effective MPC design for nonlinear systems (Sec.~\ref{sec-nonlin}). The problem of periodic optimal control with inexact models is also addressed in \cite{mirasierra2023} using periodic disturbance observers and modifier adaptation. However, this implementation utilizes knowledge of gradients of the actual system.

Our approach merges the principles of disturbance observers and repetitive control. By augmenting the MPC nominal model with a lifted
periodic disturbance, our approach extends RC to nonlinear systems with constraints. Using an observer (instead of direct updates) allows users to balance noise reduction and convergence speed. By incorporating offset-free MPC techniques and design principles into RC, we ensure perfect asymptotic tracking of periodic signals despite model mismatches.

\textbf{Outline:} 
We begin by describing the problem setup (Sec.~\ref{sec:setup}). 
Then, we present the proposed periodic disturbance observer (Sec.~\ref{sec:periodic_observer}) and the corresponding periodic tracking MPC (Sec.~\ref{sec:MPC}), including convergence guarantees (Thm.~\ref{thm:offsetfree}). 
While this exposition considers linear prediction models for simplicity, we also discuss how the method naturally generalizes to nonlinear prediction models (Sec.~\ref{sec-nonlin}). Lastly, we provide results from simulation and hardware experiments (Sec.~\ref{sec:experiments}) and present our conclusions (Sec.~\ref{sec:conclusion}).

\section{Problem Setup}\label{sec:setup}
\textbf{Notation:}
We denote the quadratic norm with respect to a positive definite matrix $Q = Q^\top$ by $\| x \|_Q^2\defeq x^\top Qx$. Non-negative integers are denoted by $\Z_{\geq 0}$, positive integers by $\Z_{>0}$, and integers in the interval $[a,b]$ by $\Z_{[a,b]}$. The $n \times n$ identity matrix is denoted by $I_n$. The spectrum of a matrix $A$ is denoted by $\sigma(A)$. The Kronecker product between matrices $A$ and $B$ is denoted by $A \otimes B$.

We consider a discrete-time nonlinear system
\begin{align} \label{eq:plant}
\begin{split}
& \xf(t+1)=f^{\mathrm{f}}\left(\xf(t), u(t)\right),\\
& \yf(t)=g^{\mathrm{f}}\left(\xf(t)\right), \\
& z(t) = H\yf(t),
\end{split}
\end{align}
where $\xf \in \R^{n}$ represents the system state, $u \in \mathbb{R}^{n_u}$ the control input, and $\yf \in \mathbb{R}^{n_y}$ the output measured at each time $t \in \Z_{\geq 0}$. The functions $f^{\mathrm{f}}$ and $g^{\mathrm{f}}$ are assumed to be unknown. %
The controlled variable, $z \in \mathbb{R}^{n_{r}}$, is a linear combination of the measured outputs where, without loss of generality, we assume $H$ has full row rank ($n_r \leq n_y$). 

The primary objective for the controlled variable $z$ is to track a periodic reference signal $r(t) \in \R^{n_r}$ for all time $t \in \Z_{\geq 0}$. We denote the reference period as $N \in \Z_{>0}$ such that periodicity of $r$ implies $r(t+N)=r(t)$.

We consider a linear time-invariant (LTI) nominal model %
\begin{align} \label{eq:nominal}
\begin{split}
    x(t+1) = A x(t) + B u(t), \; \; y(t) = C x(t),
\end{split}
\end{align}
with output $y \in \mathbb{R}^{n_y}$ and state $x \in \mathbb{R}^{n_x}$, allowing for $n_x$ to be different from $n$. We assume that $(A, B)$ is controllable, $(A, C)$ is observable, and $C$ has full row rank. The model is subject to the constraints
$x(t) \in \mathcal{X}, \; u(t) \in \mathcal{U}, \; \forall t \in \Z_{\geq 0},$
where the sets $\mathcal{X}$ and $\mathcal{U}$ are assumed to be compact.

The goal is to design an MPC scheme where the controlled variable asymptotically converges to the periodic ref., i.e.,
    $$\lim _{t \rightarrow \infty}\|z(t) - r(t)\|=0.$$
Hence, we assume there exist control inputs such that the system \eqref{eq:plant} can track the reference while satisfying constraints.

To this end, we design a linear observer that estimates periodic disturbances (Sec.~\ref{sec:periodic_observer}). We then combine it with a tracking MPC formulation and establish convergence guarantees (Sec.~\ref{sec:MPC}). While we initially consider an LTI model to streamline the exposition, we also extend the method for application with nonlinear models (Sec.~\ref{sec-nonlin}).

\section{Periodic Disturbance Observer Design}
\label{sec:periodic_observer}
In this section, we introduce a simple linear observer~\eqref{eq:estimator} to estimate periodic disturbances. These estimated disturbances should compensate for the deterministic model mismatch. In particular, we first present an augmented model \eqref{eq:LDO}, discuss its observability (Prop.~\ref{prop:augObs}), and end by characterizing the observer's convergence properties (Prop.~\ref{prop:steadyState}).

To capture the model mismatch of the true system \eqref{eq:plant} with respect to the nominal model \eqref{eq:nominal} throughout the period $N$, we estimate a `lifted' disturbance $\bd \in \R^{\nd N}$. Specifically, the lifted disturbance corresponds to $N$ disturbances 
\begin{equation} \label{eq:Ndists}
    \bd(t) = \begin{bmatrix} {d_{0}(t)}^\top & {d_1(t)}^\top & \cdots & {d_{N-1}(t)}^\top \end{bmatrix}^\top,
\end{equation} where each $d_{k}(t) \in \mathbb{R}^{\nd}$ represents the disturbance prediction computed at time $t$ for the expected disturbance at $k$ time steps in the future, for $t \in \Z_{\geq 0}, \, k\in \Z_{[0, N - 1]}$. 

Our goal will be to augment the nominal model \eqref{eq:nominal} with these periodic disturbances in such a way that the augmented model is observable, i.e., we can estimate both the state and disturbance vectors online. For this, we introduce the matrices $\Bd \in \R^{n_x \times \nd}$ and $\Cd \in \R^{n_y \times \nd}$ as design choices for how the disturbances should act on the state and output, respectively. Augmenting the nominal model \eqref{eq:nominal} with the lifted disturbance \eqref{eq:Ndists} yields
\begin{align} \label{eq:LDO}
\begin{split} 
    \begin{bmatrix}
    x(t+1) \\
    \bd(t+1)
    \end{bmatrix} = \begin{bmatrix}
    A & \Bd \sel \\
    0 & \Sd
    \end{bmatrix} &\begin{bmatrix}
    x(t) \\
    \bd(t)
    \end{bmatrix} + \begin{bmatrix}
    B \\
    0
    \end{bmatrix} u(t),\\
    y(t) = \begin{bmatrix}
        C & \Cd \sel
    \end{bmatrix} &\begin{bmatrix}
        x(t) \\
        \bd(t)
    \end{bmatrix} ,
\end{split}
\end{align}
where $\sel = \begin{bmatrix} I_{\nd} & 0 & \cdots & 0  \end{bmatrix} \in \mathbb{R}^{\nd \times \nd N}$ is a selection matrix that picks out the current (first) disturbance, and $\Sd$ advances the disturbance prediction by one time step using the cyclic forward shift matrix $S \in \mathbb{R}^{N \times N}$, defined as
\begin{equation} \label{eq:sd}
    S= \begin{bmatrix}
    0 & 1 & 0 & 0 \\
    \vspace{-6.5mm} \\ %
    \vdots & 0 & \ddots & 0 \\
    0 & 0 & 0 & 1 \\
    1 & 0 & \cdots & 0\\
    \end{bmatrix},
    \quad 
    \Sd = S \otimes I_{\nd}.
\end{equation}
In particular, we have that $\Sd^N = I_{\nd N}$. Due to the block structure of matrix $\Sd$, all of its eigenvalues $\lambda_k=e^{i 2 \pi k / N}, \, k\in \mathbb{Z}_{[0, N - 1]}$ lie on the unit circle, i.e., $\vert \lambda\vert =1$, and have algebraic and geometric multiplicity of $\nd$. The case $N=1$ corresponds to a constant disturbance, which recovers the offset-free MPC disturbance model \cite{badgwellDisturbanceModelDesign2002, pannocchia2003, maederLinearOffsetfreeModel2009, pannocchiaOffsetfreeMPCExplained2015}.

Next, we design an observer to estimate the state $x$ and the disturbance $\bd$ of the augmented model \eqref{eq:LDO}.
The following proposition clarifies when this model is observable.
\begin{proposition} \label{prop:augObs}
The augmented system \eqref{eq:LDO} is observable if and only if %
\begin{equation} \label{eq:obsRank}
\operatorname{rank}  \begin{bmatrix}
    A - \lambda I_{n_x} & \Bd \\
    C & \Cd
\end{bmatrix}  = n_x + \nd, \quad \text{for all } \lambda \in \sigma(\Sd).
\end{equation}
\end{proposition}
\begin{proof}
The proof is provided in the appendix. 
\end{proof}

When Prop.~\ref{prop:augObs} holds, i.e., \eqref{eq:LDO} is observable, the periodic disturbance $\bd$ and the state $x$ can be uniquely reconstructed from a trajectory of $y$ and $u$. This ensures that a linear observer can be designed to estimate $\bd$, $x$. In turn, these estimates will be used in the ensuing predictions of the model to enable the estimated output to converge to the true output. %

Therefore, we must choose $\Bd,\Cd$ such that the observability condition~\eqref{eq:obsRank} holds. The following remark discusses how to select disturbance models to satisfy this condition. 
\begin{remark} \label{rmk:commonDO}
Suppose for simplicity that the eigenvalues of $A$ and $\Sd$ are distinct\footnote{Disjoint spectra $\sigma(A) \cap \sigma(\Sd) = \emptyset$ are expected in general for random matrices $A$. Otherwise, $\Bd$ should be chosen such that $A-\Bd C$ has distinct eigenvalues from $\Sd$, e.g., using pole-placement.}. 
Then a simple choice of a disturbance model consists of an output disturbance, $\Bd=0$, $\Cd=I_{\nd}$, which satisfies condition~\eqref{eq:obsRank} (cf. \cite[Remark~2]{maederLinearOffsetfreeModel2009}). Alternatively, a pure input disturbance can be modeled by choosing $\Cd=0$, and condition~\eqref{eq:obsRank} reduces to choosing $\Bd$ such that $\det(C(A - \lambda I_{n_x})^{-1}\Bd)\neq 0$ $\forall \lambda\in\sigma(\Sd)$.

The special case of full state measurement, i.e., $C=I_{n_x}$, enables a simple design using %
$\Bd=I_{n_x},\, \Cd=0$, which can even be directly applied to nonlinear models -- see Sec.~\ref{sec-nonlin}.
\end{remark} More guidelines and existing results for the choice of a disturbance model can be found in~\cite{badgwellDisturbanceModelDesign2002,pannocchia2003,maederLinearOffsetfreeModel2009}.

To estimate the state and disturbance vectors online, we design a simple Luenberger observer:
\begin{align} \label{eq:estimator}
\begin{split}
\begin{bmatrix}
\xhat(t+1) \\
\bdhat(t+1)
\end{bmatrix} &= \begin{bmatrix}
A & \Bd \sel \\
0 & \Sd
\end{bmatrix} \begin{bmatrix}
\xhat(t) \\
\bdhat(t)
\end{bmatrix} + \begin{bmatrix}
B \\
0
\end{bmatrix} u(t)\\
& + \begin{bmatrix}
L_x \\
L_d
\end{bmatrix} \left(-\yf(t) + C \xhat(t) + \Cd \sel \bdhat(t) \right).
\end{split}
\end{align}
Given observability, %
we can design a stable estimator \eqref{eq:estimator} using standard techniques, e.g., pole placement or Kalman filtering. The design of $L_x, L_d$ allows users to balance noise reduction against faster estimator convergence.%

When the input and output signals become periodic\footnote{This behavior is expected in cases where the MPC yields bounded closed-loop trajectories (see Sec.~\ref{sec:experiments}).}, the observer converges to a periodic trajectory, characterized in the following proposition.%
\begin{proposition} \label{prop:steadyState}
Suppose the input and output signal are asymptotically $N$-periodic, i.e., $u(t+N) = u(t)$ and $y(t+N) = y(t)$ for $t\rightarrow\infty$. Then, the estimator \eqref{eq:estimator} converges to periodic trajectories $\bxhat$, $\bdhat$ that satisfy
\begin{align} \label{eq:ObsSSR}
\begin{bmatrix}
    A_N - \Sx & B_N \\
    C_N & 0 
\end{bmatrix}
\begin{bmatrix}
    \bxhat(t) \\
    \bu(t)
\end{bmatrix} = \begin{bmatrix}
    - {\Bd}_N \bdhat (t)\\
    \byf(t) -  {\Cd}_N \bdhat(t)
\end{bmatrix},
\end{align}
where we denote $M_N \defeq I_N \otimes M$ for any matrix $M$, define $\Sx \defeq S \otimes I_{n_x}$ as the block-cyclic permutation matrix, and introduce $u_k(t) \defeq u(t+k)$ and $\yf_k(t) \defeq \yf(t+k)$ to express the periodic trajectories in the limit $t \to \infty$ as
\begin{align*}
    \bu(t) &= \begin{bmatrix} u_0(t)^\top & u_{1}(t)^\top & \cdots & u_{N-1}(t)^\top \end{bmatrix}^\top,\\
    \byf(t) &= \begin{bmatrix} {\yf_0}(t)^\top & {\yf_{1}}(t)^\top & \cdots & {\yf_{N-1}}(t)^\top \end{bmatrix}^\top,\\
    \bxhat(t) &= \begin{bmatrix} {\xhat_0}(t)^\top & {\xhat_{1}}(t)^\top & \cdots & {\xhat_{N-1}}(t)^\top \end{bmatrix}^\top.
\end{align*}
\end{proposition}
\begin{proof}
Periodic input and output with a stable observer~\eqref{eq:estimator} implies that $\xhat$ and $\bdhat$ asymptotically converge to a periodic trajectory with the same period \cite{haddleton1994}. Thus, as $t \to \infty$, we have $\bdhat(t+N) = \bdhat(t) = \Sd^N \bdhat(t)$. 

Defining $e(t) = -\yf(t) + C \xhat(t) + \Cd  \dhat_0(t)$ and focusing on the bottom row of~\eqref{eq:estimator} we obtain
\begin{align*}
    &\bdhat(t+N) = \Sd^N \bdhat(t) + \sum_{j=0}^{N-1} \Sd^{N-1 - j} L_d e(t+j)\\
   &\iff 0=   \begin{bmatrix}
\Sd^{N-1} L_d &  \cdots & \Sd L_d & L_d
   \end{bmatrix}    
   \begin{bmatrix}
       e(t)\\
       \vdots \\
       e(t+N-1)
   \end{bmatrix}.
\end{align*}
The $N \nd \times N n_y$ matrix above is the controllability matrix for $(\Sd, L_d)$, which has full rank since the observer~\eqref{eq:estimator} is stable -- see Prop.~\ref{prop:rankObs} in the appendix. Inverting, we have
\begin{align}
0 &= e(i) = -\yf(i) + C \xhat(i) + \Cd \dhat(i), \; \forall i \in \Z_{[t, t+N-1]}. \label{eq:zero_err}
\end{align}Substituting \eqref{eq:zero_err} into the top row of \eqref{eq:estimator}, we obtain
\begin{align} \label{eq:per_state}
\xhat(i+1) = A \xhat(i) + \Bd \dhat(i) + B u(i).
\end{align}
Combining \eqref{eq:zero_err} and \eqref{eq:per_state} leads to \eqref{eq:ObsSSR}. 
\qedhere
\end{proof}

Prop.~\ref{prop:steadyState} generalizes the theoretical results in \cite[Prop.~3]{maederLinearOffsetfreeModel2009} and allows us to provide convergence guarantees for the MPC in the next section.

\section{Periodic Model Predictive Control ($\Pi$-MPC)} \label{sec:MPC}
We now present an MPC scheme that leverages the observer \eqref{eq:estimator} to asymptotically achieve zero tracking error for periodic references. The observer provides estimates $\bdhat$ that are used to compute targets $\bxbar$, $\bubar$ \eqref{eq:ref_sel} for the state and input, respectively. The MPC problem \eqref{eq:LDOMPC} is then formulated to minimize deviations from these targets, ensuring the estimated controlled variable converges to the reference.

\subsection{Target computation}
We compute the state and input targets
\begin{align*}
    \bxbar(t) &= \begin{bmatrix} \xbar_{0}(t)^\top & \xbar_{1}(t)^\top & \cdots & \xbar_{N-1}(t)^\top \end{bmatrix}^\top, \\
    \bubar(t) &= \begin{bmatrix} \ubar_{0}(t)^\top & \ubar_{1}(t)^\top & \cdots & \ubar_{N-1}(t)^\top \end{bmatrix}^\top,
\end{align*}
at time $t$ using%
\begin{equation} \label{eq:ref_sel}
\begin{bmatrix}
    A_N - \Sx & B_N \\
    H_N C_N & 0 
\end{bmatrix}
\begin{bmatrix}
    \bxbar(t) \\
    \bubar(t)
\end{bmatrix} = \begin{bmatrix}
    - {\Bd}_N \bdhat(t)\\
    \mathbf{r}(t) -  H_N {\Cd}_N \bdhat(t)
\end{bmatrix},
\end{equation}
where $H_N = I_N \otimes H$. The targets correspond to the trajectory $\bxbar, \bubar$ that achieves reference tracking for a given disturbance estimate $\bdhat(t)$, analogous to \eqref{eq:ObsSSR} in Prop.~\ref{prop:steadyState}. 

Consequently, we assume that
\begin{equation} \label{eq:wellposed}
    \operatorname{rank} \begin{bmatrix}
    A - \lambda I_{n_x} & B \\
    HC & 0
\end{bmatrix}= n_x + n_r, \quad \forall \lambda \in \sigma(\Sd).
\end{equation}
This condition\footnote{Condition \eqref{eq:wellposed} implies $n_r \leq n_u$. When $n_r < n_u$, \eqref{eq:ref_sel} is under-determined and we use the minimal norm solution.}
ensures that the target computation~\eqref{eq:ref_sel} is feasible for any disturbance estimate $\bdhat$ and any reference $\mathbf{r}$, see~\cite{maederOffsetfreeReferenceTracking2010}. 
Condition~\eqref{eq:wellposed} requires the transmission zeros from $u$ to $z$ to be distinct from $\sigma(\Sd)$, which holds generically for random matrices and is a necessary condition for tracking and disturbance rejection for the LTI system~~\cite[Lemma~1]{davison1976}.

Having to compute the targets $\bxbar(t), \bubar(t)$ at each time step may be computationally expensive or undesirable in practice. This limitation can be addressed using the alternative $\Pi$-MPC formulation \eqref{eq:MPC_alternative_nonlinear} discussed in Sec.~\ref{sec-nonlin}. 

\subsection{$\Pi$-MPC formulation}
We now formulate the MPC with horizon length $L$ as 
\begin{subequations} \label{eq:LDOMPC}
\begin{align}
\min_{u_0, \ldots, u_{L-1}} \,& 
\| x_L - \xbar_{L}(t) \|_P^2\\
&+ \sum_{k=0}^{L-1} \| x_{k} - \xbar_{k}(t) \|_Q^2 
+ \| u_{k} - \ubar_{k}(t) \|_R^2 \nonumber\\
\text{s.t.} \quad & x_{k+1} = A x_{k} + B u_{k} + \Bd \dhat_k (t), \quad k \in \Z_{[0,L-1]}, \nonumber\\
& x_0 = \hat{x}(t), \; x_{k} \in \mathcal{X}, \; k \in \Z_{[1,L]}, \label{xconst}\\
& u_{k} \in \mathcal{U}, \quad k \in \Z_{[0,L-1]}, \label{uconst}
\end{align}
\end{subequations}
where $Q \succeq 0$, $R \succ 0$, $(A, Q)$ is detectable, the terminal cost $P$ is chosen using a linear quadratic regulator (LQR), and we assume $L<N$ for simplicity\footnote{Otherwise we repeat the periodic targets to fill the horizon length.}. %
At each time $t$, the observer provides the state estimate $\xhat(t)$ and the estimates $\dhat_k (t)$ for the expected disturbance $k$ time steps ahead. We denote the optimal solution to \eqref{eq:LDOMPC} with a star ($^\star$).

The following algorithm recaps the offline design of the periodic MPC scheme ($\Pi$-MPC), which includes the disturbance model, the observer, and the LQR design. 

\begin{algorithm}[H]
\caption{$\Pi$-MPC (Offline Design)}
\begin{algorithmic}
    \State Given LTI model \eqref{eq:nominal} and reference period $N$
    \State Choose $\Bd, \Cd$ such that \eqref{eq:obsRank} holds
    \State Design observer gains $L_x, L_d$ in \eqref{eq:estimator} (e.g., Kalman filter)
    \State Choose $Q \succeq 0$, $R \succ 0$, with $(A,Q)$ detectable
    \State Compute $P \succ 0$ using LQR    
\end{algorithmic}
\label{alg:offline}
\end{algorithm}
The following algorithm summarizes the closed-loop control of \eqref{eq:plant} with the observer \eqref{eq:estimator} and MPC \eqref{eq:ref_sel}, \eqref{eq:LDOMPC}.

\begin{algorithm}[H]
\caption{$\Pi$-MPC (Online Operation)}
At $t=0$: Initialize $\hat{x}(0), \bdhat(0)$
\begin{algorithmic}
\For{each time $t \in \Z_{\geq 0}$}
\State Compute targets $\bxbar(t), \bubar(t)$ according to~\eqref{eq:ref_sel}%
\State Solve Problem~\eqref{eq:LDOMPC} and apply $u(t) = u^\star_0$
\State Measure $\yf(t)$ and estimate $\hat{x}(t+1)$, $\bdhat(t+1)$ using~\eqref{eq:estimator}
\EndFor
\end{algorithmic}
\label{alg:online}
\end{algorithm}
The considered problem could also be addressed with the disturbance observer design from \cite{maederOffsetfreeReferenceTracking2010}, which decomposes general linear signals using eigenmodes. Naively applying eigenmode decomposition methods to periodic problems fails to account for sparsity, thus complicating observer design and MPC target calculations as complexity scales with period length. In contrast, the proposed periodic disturbance observer provides a simple and efficient MPC implementation through sequential disturbances in time. %

\subsection{Convergence analysis}

The following theorem provides the main theoretical result of this paper, showing that the closed-loop system resulting from Alg.~\ref{alg:online} asymptotically achieves zero tracking error.

\begin{theorem}
\label{thm:offsetfree}
Assume that the MPC problem \eqref{eq:LDOMPC} is feasible for all $t \in \Z_{\geq 0}$, the constraints~\eqref{xconst}-\eqref{uconst} are inactive %
after some time $t \geq j$ with $j \in \Z_{\geq 0}$, and the closed-loop system in Alg.~\ref{alg:online} converges to a periodic trajectory denoted by $\bu(t), \byf(t)$.
Then zero tracking error is achieved asymptotically, i.e., $\|z(t) - r(t)\| \to 0$ as $t\rightarrow \infty$.
\end{theorem}

\begin{proof}
The following proof extends the arguments in \cite[Thm. 1]{maederLinearOffsetfreeModel2009} to the periodic reference tracking case.

As $t \to \infty$, Prop.~\ref{prop:steadyState} implies that the $N$-periodic sequences $\bu$, $\byf$, $\bdhat$, and $\bxhat$ satisfy \eqref{eq:ObsSSR}. 
By definition, the targets $\bxbar$ and $\bubar$ satisfy \eqref{eq:ref_sel}.
Left multiplying the second row of \eqref{eq:ObsSSR} by $H_N$, and then subtracting \eqref{eq:ref_sel} from \eqref{eq:ObsSSR} we obtain
\begin{align} \label{eq:errorSSR}
    \begin{bmatrix}
        A_N - \Sx & B_N \\
        H_N C_N & 0 
    \end{bmatrix}
    \begin{bmatrix}
        \bxhat(t) - \bxbar(t) \\
        \bu(t) - \bubar(t) 
    \end{bmatrix} &= \begin{bmatrix}
        0 \\
        H_N \byf(t)  - \mathbf{r}(t)
    \end{bmatrix}.
\end{align}

Consider a change of variables in the MPC problem \eqref{eq:LDOMPC} to $\delta x_k = x_k - \xbar_k(t)$ and $\delta u_k = u_k - \ubar_k(t)$. Since the constraints are inactive for $t \geq j$, the MPC \eqref{eq:LDOMPC} is equivalent to
\begin{align*}
\min_{\delta u_0, \ldots, \delta u_{L-1}} \, & 
\| \delta x_L \|_P^2 +
\sum_{k=0}^{L-1} \| \delta x_{k} \|_Q^2 + \| \delta u_{k} \|_R^2 \nonumber\\
\text{s.t.} \quad & \delta x_{k+1} = A \delta x_{k} + B \delta u_{k}, \quad k \in \Z_{[0,L-1]}, \nonumber\\
& \delta x_0 = \xhat(t) - \xbar_{0}(t).
\end{align*}

Since the terminal cost $P$ is chosen based on the LQR, the optimal input is given by the unconstrained optimal LQR, i.e., $\delta u_0^\star(t) = K \delta x(t) = K(\xhat(t) - \xbar_{0}(t))$. Furthermore, given $(A,B)$ stabilizable and $(A,Q)$ detectable, $A+BK$ is (Schur) stable. Thus, as $t\to \infty$, 
the top row of \eqref{eq:errorSSR} implies %
\begin{align*}
\left( (A + B K)_N - \Sx  \right) \left( \bxhat(t) - \bxbar(t) \right) &= 0,
\end{align*}
where the left matrix is invertible by stability of $A + B K$.
Finally, since $\bxhat(t) - \bxbar(t) = 0$, the bottom row of \eqref{eq:errorSSR} yields 
$$\lim_{t\rightarrow\infty}\|Hy(t)-r(t)\|=\lim_{t\rightarrow\infty}\|z(t)-r(t)\|=0. \qedhere$$

\end{proof}

\begin{figure*}[t]
  \centering
  \includegraphics[width=0.85\textwidth]{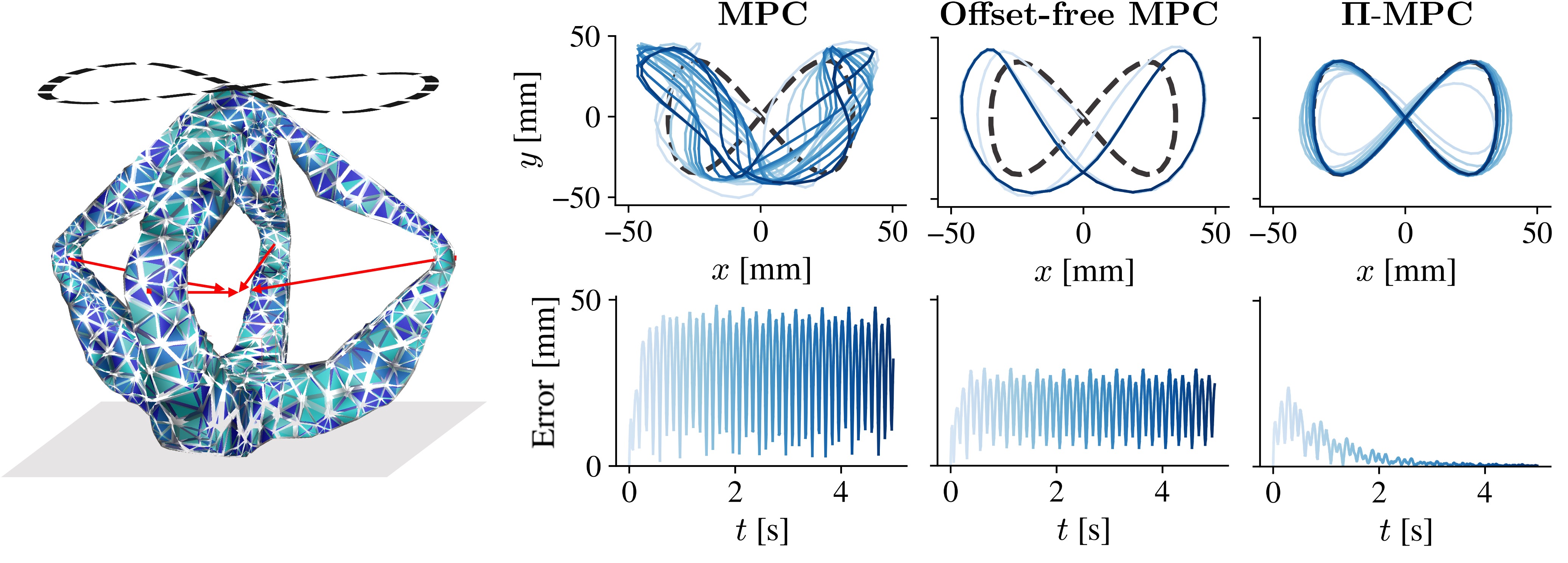}
  \caption{Left: Picture of the Diamond robot mesh. The reference figure-eight trajectory for the tip is shown with a dashed black line. The red arrows indicate the actuator inputs, i.e., applied forces at the elbows. Right: Simulation results illustrate tracking performance over ten periods for a high-frequency, 2D figure-eight trajectory. We compare a standard MPC scheme (left), offset-free MPC (center), and our proposed MPC with a periodic disturbance observer ($\Pi$-MPC) (right), all shown in blue. The bottom plots show the tracking error over time. The shading indicates time progression, with lighter shades representing earlier portions of the trajectory. The dashed black line represents the reference trajectory.}
  \vspace{-0.5cm}
  \label{fig:softrobot}
\end{figure*}

\section{Implementation and convergence for nonlinear MPC} \label{sec-nonlin}
In the following, we discuss how to generalize the presented design and analysis to more practical nonlinear MPC problems. 
Specifically, we discuss convergence guarantees with general nonlinear disturbance observers (Sec.~\ref{sec-nonlin_analysis}), a nonlinear MPC implementation which does not require the explicit computation of the targets $\bxbar$, $\bubar$ (Sec.~\ref{sec-nonlin-MPC}), and simple designs in case of state measurement (Sec.~\ref{sec-nonlin-design}). 
\subsection{Nonlinear disturbance observer - convergence analysis}
\label{sec-nonlin_analysis}
The augmented linear model~\eqref{eq:LDO} is generalized to a nonlinear model
\begin{align} \label{eq:NDO}
\begin{split} 
    \begin{bmatrix}
    x(t+1) \\
    \bd(t+1)
    \end{bmatrix} &= 
    \begin{bmatrix}
    f(x(t),u(t),\sel \bd(t))\\
    \Sd \bd(t)
    \end{bmatrix},\\ 
    y(t) &= h(x(t),\sel \bd(t)).
\end{split}
\end{align}
The observer~\eqref{eq:estimator} generalizes to
\begin{align} \label{eq:estimator_nonlinear}
\begin{split}
\begin{bmatrix}
\xhat(t+1)\\
\bdhat(t+1)
\end{bmatrix} =&
\begin{bmatrix}
    f(\xhat(t),u(t),\sel \bdhat(t))\\
    \Sd \bdhat(t)
    \end{bmatrix},\\  
 &+ \ell( \yf(t), \xhat(t),\sel \bdhat(t) ).
\end{split}
\end{align}
The analysis of this nonlinear periodic disturbance model is based on a combination of the arguments for nonlinear disturbance observers in~\cite{morariNonlinearOffsetfreeModel2012} and the proposed periodic disturbance observer (Sec.~\ref{sec:periodic_observer}--\ref{sec:MPC}). 
Similar to Prop.~\ref{prop:steadyState}, if a stable observer is designed and the input and output trajectories converge to a periodic trajectory, then the converged estimates satisfy $\yf(t)=h(\xhat(t),\dhat_0(t))$ (cf.~\cite[Thm.~4]{morariNonlinearOffsetfreeModel2012}).
Then, asymptotic tracking of the reference can be ensured if the MPC design ensures $\lim_{t\rightarrow\infty}\|H h(\xhat(t),\dhat_0(t)))-r(t)\|=0$ whenever the prediction model is exact.

\subsection{Nonlinear MPC design}
\label{sec-nonlin-MPC}
Application of the MPC scheme (Alg.~\ref{alg:online}) with the nonlinear model~\eqref{eq:NDO} is complicated by two factors:
(i) computing a periodic target $\bxbar,\bubar$ \eqref{eq:ref_sel} is computationally expensive;
(ii) relating the MPC scheme to an LQR or designing a suitable terminal penalty $P$ becomes non-trivial. 
Instead, the following MPC formulation from~\cite{kohlerConstrainedNonlinearOutput2022} provides a simple solution: 
\begin{align}
\label{eq:MPC_alternative_nonlinear}
\begin{split}
\min_{u_0, \ldots, u_{L-1}} \quad & \sum_{k=0}^{L-1} \| z_{k} - r_{k} \|_{Q_z}^2 + \| u_{k} - u_{k-N} \|_R^2 \\
\text{s.t.} \quad & x_{k+1} = f(x_{k},u_{k},\dhat_k (t)), \quad x_0 = \hat{x} (t), \\ %
& y_{k} = h(x_k,\dhat_k (t)),\\
& z_{k} = H y_k, \quad k \in \Z_{[0,L-1]},\\
& x_{k} \in \mathcal{X}, \quad k \in \Z_{[1,L]}, \\
& u_{k} \in \mathcal{U}, \quad k \in \Z_{[0,L-1]},
\end{split}
\end{align}
with $Q_z$ positive definite. This formulation directly minimizes the error with respect to the reference and regularizes the input by penalizing non-periodicity. Under suitable technical conditions (involving stabilizability, detectability, and non-resonance), this MPC scheme satisfies the desired tracking properties, i.e., the reference is asymptotically tracked when the prediction model is exact if the horizon $L$ is chosen sufficiently large, see~\cite[Sec.~IV]{kohlerConstrainedNonlinearOutput2022} for details.

\subsection{Nonlinear observer design}
\label{sec-nonlin-design}
The design of nonlinear observers with guaranteed stability is generally challenging, but promising results can often be obtained using a simple extended Kalman filter (EKF).
In the special case of full state measurements $y(t)=x(t)$, a simple linear observer can be designed for an additive disturbance model $x(t+1)=f(x(t),u(t))+\sel \bd(t)$. 
The disturbance observer is given by
\begin{align}
\label{eq:nonlinear_observer_simple}
\begin{split}
&\bdhat(t+1)=S_d\bdhat(t)\\
&+\sel^\top  L_d(f(x(t),u(t))+ \dhat_0(t) -x(t+1)),
\end{split}
\end{align}
which essentially corresponds to an update of $\hat{d}_0(t)$ based on the difference between the prediction and the new measured state. %
Here, the matrix $L_d$ should be chosen such that $I+L_d$ is Schur stable, e.g., $L_d=-\lambda I_{\nd}$, $\lambda\in(0,1)$.

Overall, the design from Sec.~\ref{sec:periodic_observer}--\ref{sec:MPC} is naturally extended to nonlinear MPC while maintaining the simplicity and theoretical guarantees for asymptotically perfect tracking.

\section{Experiments} \label{sec:experiments}
We demonstrate our approach's broad applicability on two challenging robotic systems with significant model mismatch, leveraging the MPC outlined in Sec.~\ref{sec-nonlin-MPC}. First, we validate the $\Pi$-MPC scheme in simulation on an underactuated soft robot with nearly 10,000 degrees of freedom, where we use a simple linear 6-dimensional prediction model. Next, we apply our method to a real-world miniature race car and show its ability to achieve near-perfect tracking of a given reference with a simple kinematic bicycle model.

\begin{figure*}[t]
  \centering
  \includegraphics[width=0.9\textwidth]{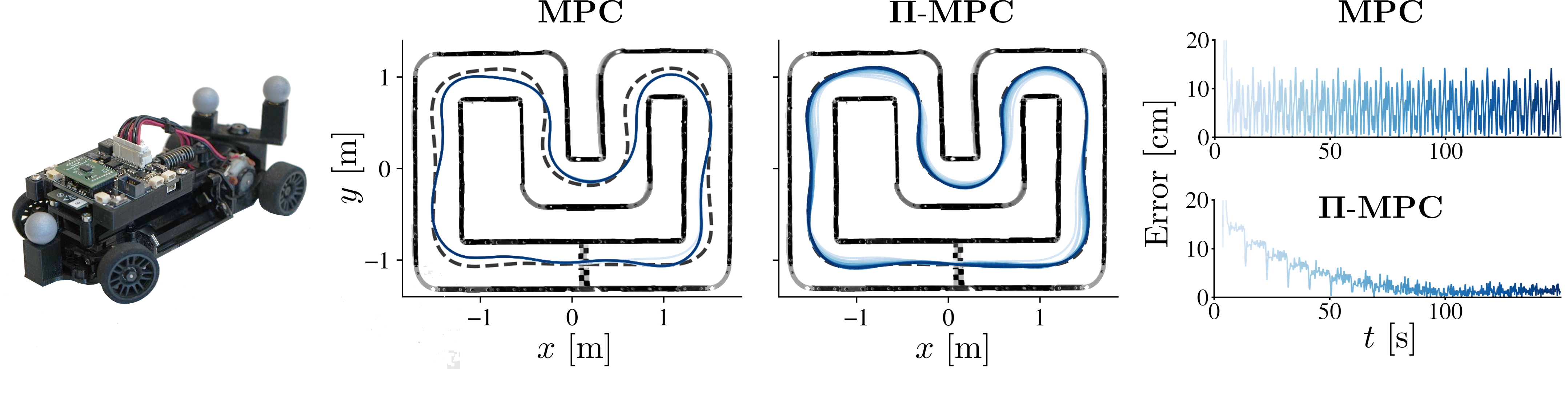}
  \caption{%
  Experimental results illustrate tracking with a race car. We compare a standard MPC scheme and the same MPC with the proposed periodic disturbance observer ($\Pi$-MPC), all shown in blue. The shading indicates time progression, with lighter shades representing earlier portions of the trajectory. The dashed black line represents the reference trajectory.}
  \vspace{-0.5cm}
  \label{fig:car_big}
\end{figure*}

\subsection{Soft robot finite element simulation}

We now apply our approach in simulation to the `Diamond' soft robot (shown in Fig.~\ref{fig:softrobot}). We compare:
\begin{enumerate}
    \item an MPC scheme using only the nominal model,
    \item the same MPC equipped with a constant disturbance observer (Offset-free MPC, OF-MPC \cite{badgwellDisturbanceModelDesign2002, pannocchia2003, maederLinearOffsetfreeModel2009, pannocchiaOffsetfreeMPCExplained2015}),
    \item the same MPC with the proposed periodic disturbance observer ($\Pi$-MPC).
\end{enumerate}

We demonstrate that $\Pi$-MPC asymptotically eliminates tracking error, achieving perfect tracking on a challenging, high-frequency periodic control task.

\begin{table}[t]
\centering
\caption{Quantitative Evaluation of Tracking Errors in Soft-Robot Sim.}
\label{tab:softrobot}
\begin{tabular}{l@{\hskip 3pt}p{1.3cm}p{1.3cm}p{1.3cm}p{1.3cm}}
\hline
& \multicolumn{2}{c}{\textbf{10\textsuperscript{th} Period}} & \multicolumn{2}{c}{\textbf{50\textsuperscript{th} Period}} \\
\cline{2-5}
 & \textbf{Avg.} & \textbf{Max.} & \textbf{Avg.} & \textbf{Max.} \\
\hline
\textbf{MPC}    & 30.5 \si{mm}  & 46.8 \si{mm} & 30.0 \si{mm} & 47.4 \si{mm} \\
\textbf{OF-MPC} & 18.6 \si{mm} & 28.7 \si{mm} & 18.6 \si{mm} & 28.7 \si{mm} \\
\textbf{$\mathbf{\Pi}$-MPC} & \textbf{0.26} \si{mm} & \textbf{0.52} \si{mm} & \textbf{0.004}$\,$\si{mm} & \textbf{0.008}$\,$\si{mm} \\
\hline
\end{tabular}
\vspace*{-1.6\baselineskip}
\end{table}

We conduct simulations through the SOFA finite-element-based physics simulator \cite{allard2007sofa}. The Diamond robot mesh is available in the \textit{SoftRobots} plugin \cite{coevoet2017software}. The robot features four actuators, shown in red in Fig.~\ref{fig:softrobot}, each pulling at an elbow. The robot's physical parameters match those reported in \cite{alora2023data}. The output measurement, $y^\text{f}$, is the $xyz$ position of the robot's tip at time $k$ and at time $k-1$. Following the framework outlined in \cite{alora2023data}, we regress a continuous-time linear model by collecting decaying trajectories and identifying a low-dimensional subspace with latent coordinates $x(t)\in\mathbb{R}^6$.

For MPC, we discretize the model with a time step of $dt = 0.01$ \si{s}. To achieve real-time control, we select a prediction horizon of $L = 15$ steps. This configuration results in solve times of approximately 3 \si{ms} across all MPC schemes. Since the learned model's $A$ matrix has no eigenvalues in common with $\Sd$, we choose a disturbance model with $\Bd=0$, $\Cd=I$, cf. Remark \ref{rmk:commonDO}. Finally, the observer gains in \eqref{eq:estimator} are calculated using a Kalman filter, with resulting magnitudes of closed-loop eigenvalues between $0.68$ and $0.98$.

The control task has the robot's end effector tracing a figure-eight along the $xy$--plane, with freedom in the $z$-direction. The figure-eight trajectory has an amplitude of $35$ \si{mm} and a frequency of $2$ \si{Hz}, corresponding to a period of $N = 50$. No random noise is added to the simulation. %

Fig.~\ref{fig:softrobot} illustrates the superior tracking performance of our method relative to the baseline methods across ten periods, with quantitative results detailed in Table~\ref{tab:softrobot}. The standard MPC scheme exhibits poor closed-loop tracking performance and fails to track the desired high-frequency trajectory, as the learned linear model does not give accurate predictions of the full nonlinear FEM system. The offset-free disturbance observer improves performance but still exhibits large tracking errors as it tries to estimate a constant disturbance, despite the model mismatch being time-varying. Instead, our approach properly considers the disturbances at each point throughout the trajectory. As expected from the presented theory, $\Pi$-MPC ensures that the tracking error decays to zero asymptotically despite significant model discrepancy. In fact, after 50 periods, the peak error reduces below $1 \times 10^{-2}$ \si{mm}.

\subsection{Miniature race car experiments}

In the following, we showcase the practicality of the proposed approach in realistic conditions through hardware experiments -- see Fig.~\ref{fig:titlefig}B.
The experiments were conducted on a miniature RC car  (scale 1:28) in
combination with the CRS software framework; for details on the overall code framework and the involved hardware see~\cite{carron2023chronos}. 
The MPC uses a simple kinematic bicycle model for the car~\cite{rajamani2011vehicle}:
\begin{align*}
\dot{p}_{\mathrm{x}}=&v\cos(\psi+\beta),\quad 
\dot{p}_{\mathrm{y}}=v\sin(\psi+\beta)\\
\dot{\psi}=&v/l_{\mathrm{r}} \sin(\beta),\quad \dot{v}=a,\quad 
\beta=\arctan\left(\frac{l_{\mathrm{r}}}{l_{\mathrm{r}}+l_{\mathrm{f}}} \tan(\delta)\right)\\
x=&[p_{\mathrm{x}},p_{\mathrm{y}},\psi,v]^\top\quad 
u=[\delta,a]^\top,\quad z=[p_{\mathrm{x}},p_{\mathrm{y}}]^\top,
\end{align*}
where $p_{\mathrm{x}/\mathrm{y}}$ are the positions, $\psi$ is the heading angle, $v$ the  velocity, $\beta$ the slip, $\delta$  the steering angle, and $a$ the acceleration. The state $x$ is measured using a Qualiysis motion capture system.
As discussed in Sec.~\ref{sec-nonlin}, we use the observer~\eqref{eq:nonlinear_observer_simple}, and the design consists of choosing a gain matrix $L_d\in\mathbb{R}^{4\times 4}$ that determines convergence speed. Experiments are conducted with $L_d = -\mathrm{diag}(0.1,0.1,0.2,0.2)$.%

The periodic reference is chosen as a physically feasible trajectory on the racetrack based on past experiments.
When solving~\eqref{eq:MPC_alternative_nonlinear}, we penalize non-periodicity of $du/dt$ instead of $u$ to yield smoother operation. 
The MPC problem \eqref{eq:MPC_alternative_nonlinear} is solved online using \textit{acados}~\cite{verschueren2022acados}.
The overall implementation considers a prediction horizon of $L=40$, a period length of $N=231$, and a sampling period of $40$~\si{ms}. 

In the experiments, we compare a na\"ive MPC implementation, using only the model, with the proposed $\Pi$-MPC approach, which additionally uses the periodic disturbance observer. 
The experimental results are illustrated in Fig.~\ref{fig:car_big}. Quantitative results are summarized in Table \ref{tab:racecar}.
Both MPC formulations provide identical results in the first lap. After ten laps, $\Pi$-MPC has an average error approximately 4 times lower than the MPC baseline. Over the course of sixteen laps, the na\"ive MPC implementation continues to show peak errors over $14$ \si{cm}, while the proposed formulation reduces the peak tracking error below $3$ \si{cm}. Although the theory suggests convergence to zero error, the presence of non-deterministic effects, such as noise and delays, may result in small fluctuations.

\begin{table}[t]
\centering
\caption{Quantitative Evaluation of Tracking Errors in Race Car Exp.}
\label{tab:racecar}
\begin{tabular}{lcccc}
\hline
& \multicolumn{2}{c}{\textbf{5\textsuperscript{th} Lap}} & \multicolumn{2}{c}{\textbf{16\textsuperscript{th} Lap}} \\
\cline{2-5}
 & \textbf{Avg.} & \textbf{Max.} & \textbf{Avg.} & \textbf{Max.} \\
\hline
\textbf{MPC} & 6.19 \si{cm} & 14.21 \si{cm} & 6.15 \si{cm} & 14.23 \si{cm} \\
\textbf{$\mathbf{\Pi}$-MPC} & \textbf{5.54} \si{cm} & \textbf{8.04} \si{cm} & \textbf{1.42} \si{cm} & \textbf{2.91} \si{cm} \\
\hline
\end{tabular}
\vspace*{-1.6\baselineskip}
\end{table}

Overall, the baseline MPC exhibits significant tracking errors and oscillations caused by the model mismatch. 
In contrast, the proposed formulation achieves almost perfect tracking after a few periods. 
Notably, the proposed approach has minimal design complexity and is implemented in a modular way in addition to an existing MPC implementation.

\section{Conclusion}
\label{sec:conclusion}
Our work shows that including a periodic disturbance observer in MPC is a simple and effective method to remove tracking errors for periodic references. Specifically, we have shown that the proposed $\Pi$-MPC is
\begin{itemize}
    \item easily implemented on top of an existing MPC scheme with both linear or nonlinear models,
    \item characterized with theoretical guarantees, and
    \item validated numerically and experimentally to achieve minimal tracking errors, even with significant model-system mismatches.
\end{itemize}
Non-periodic problems are left for future work.

\section*{Appendix}
\begin{proof} [Proof of Prop.~\ref{prop:augObs}]
    This proof extends the results in \cite[Prop.~1]{maederLinearOffsetfreeModel2009} to periodic disturbances.
    From the Hautus observability condition \cite[p.~272]{Sontag_1998}, the observability of system \eqref{eq:LDO} is equivalent to 
    \begin{equation} \label{eq:augRank}
    \operatorname{rank}
    \begin{bmatrix}   
    A-\lambda I_{n_x} & \Bd \sel \\
    0 & \Sd - \lambda I_{\nd N} \\
    C & \Cd \sel
    \end{bmatrix} = n_x + \nd N,
    \end{equation}
    for all $\lambda \in \mathbb{C}$. 

    When $\lambda \notin \sigma(\Sd)$, we have that $\Sd - \lambda I_{\nd N}$ is full rank. %
    Furthermore, since $(A,C)$ is observable, the Hautus condition on \eqref{eq:nominal} implies $\operatorname{rank} \left( \begin{bmatrix}
    (A-\lambda I_{n_x})^\top & C^\top 
    \end{bmatrix}^\top \right) = n_x$. Thus, the left and right sides of the matrix contribute $n_x$ and $\nd N$ independent columns, respectively, and \eqref{eq:augRank} holds.

    When $\lambda \in \sigma(\Sd)$, we have $\lambda = \lambda_k=e^{i 2 \pi k / N}, \, k\in \mathbb{Z}_{[0, N - 1]}$. Since the geometric multiplicity of each $\lambda_k$ is $\nd$, the dimension of the null-space of $\Sd - \lambda_k I_{\nd N} $ is $\nd$. The rank-nullity theorem implies that $\operatorname{rank}\left(\Sd - \lambda_k I_{\nd N} \right)= \nd(N - 1)$. These $\nd(N - 1)$ columns are clearly independent from the left side of the matrix and can be removed from the Hautus condition \eqref{eq:augRank}, yielding
    $$\operatorname{rank} \begin{bmatrix}   
    A-\lambda_k I_{n_x} & \Bd \sel \\
    C & \Cd \sel
    \end{bmatrix} = n_x + \nd.$$
    Disregarding the additional zero columns introduced by multiplying $\Bd$ and $\Cd$ with $\sel$ yields the rank condition \eqref{eq:obsRank}. Thus, condition \eqref{eq:augRank} is equivalent to \eqref{eq:obsRank}. \qedhere 
\end{proof}

\begin{proposition} \label{prop:rankObs}
    Assume the observer \eqref{eq:estimator} is stable. Then the controllability matrix for the pair $(\Sd, L_d)$ is full row rank.
\end{proposition}
\begin{proof}
By stability of the observer \eqref{eq:estimator}, we have that
\begin{align*}
\det \left( \begin{bmatrix}
A - \lambda I_{n_x} + L_x C & (\Bd + L_x \Cd) \sel \\
L_d C & \Sd - \lambda I_{\nd N} + L_d \Cd \sel
\end{bmatrix} \right) \neq 0,
\end{align*}
for all unstable eigenvalues, $|\lambda| \geq 1$. Hence, for all $\lambda \in \sigma(\Sd)$, the bottom $\nd N$ rows must be full row rank, i.e.,
\begin{align*}
 \nd N &= \operatorname{rank} \left( \begin{bmatrix}
L_d C & \Sd - \lambda I_{\nd N} + L_d \Cd \sel
\end{bmatrix} \right),\\
&= \operatorname{rank} \left( \begin{bmatrix}
    L_d & \Sd - \lambda I_{\nd N} 
    \end{bmatrix} \begin{bmatrix} C & \Cd \sel \\ 0 & I_{\nd N} \end{bmatrix} 
     \right),\\
&= \operatorname{rank} \left( \begin{bmatrix}
L_d & \Sd - \lambda I_{\nd N} 
\end{bmatrix} \right).
\end{align*}
The last equality leverages the full row rank of $C$, ensuring the upper triangular matrix 
also has full row rank. The claim now follows from the Hautus Lemma \cite[Lemma 3.3.7]{Sontag_1998}. \qedhere
\end{proof}

\setstretch{0.95}
\bibliography{refs} 
\bibliographystyle{ieeetr}

\end{document}